\documentclass[11pt]{article}

\usepackage[letterpaper, left=1in, right=1in, top=1in,bottom=1in]{geometry}

\usepackage{parskip}

\usepackage{booktabs} 

\usepackage[utf8]{inputenc}

\usepackage{geometry}

\usepackage{graphpap,amscd,mathrsfs,graphicx,lscape,enumitem,dsfont,bm,url,subfigure}
\usepackage{epsfig,amstext,xspace}
\usepackage{algorithm,comment}
\usepackage{algorithmic}
\usepackage{thmtools,thm-restate}

\usepackage{algorithm,algorithmic}
\usepackage{auxiliary}
\usepackage[utf8]{inputenc} 
\usepackage[T1]{fontenc}    
\usepackage{hyperref}       
\usepackage{url}            
\usepackage{booktabs}       
\usepackage{amsfonts}       
\usepackage{nicefrac}       
\usepackage{microtype}      

\newcommand{\s}{\mathcal{S}}

\newcommand{\half}{\frac{1}{2}}
\newcommand{\nhalf}{\nicefrac{1}{2}}

\usepackage{color}              
\usepackage[suppress]{color-edits}
\addauthor{tl}{cyan}
\addauthor{vm}{green}
\addauthor{rpl}{brown}

\begin{document}
\title{Stochastic bandits robust to adversarial corruptions}

 \author{
 Thodoris Lykouris\thanks{Cornell University, \texttt{teddlyk@cs.cornell.edu}. Work supported under NSF grant CCF-1563714. Part of the work was done while the author was interning at Google.} \and Vahab Mirrokni\thanks{Google Research, \texttt{mirrokni@google.com}} \and Renato Paes Leme\thanks{Google Research, \texttt{renatoppl@google.com}}}
\date{}
\maketitle

\begin{abstract}

We introduce a new model of stochastic bandits with adversarial corruptions which aims to capture settings where most of the input follows a stochastic pattern but some fraction of it can be adversarially changed to trick the algorithm, e.g., click fraud, fake reviews and email spam.  The goal of this model is to encourage the design of bandit algorithms that (i) work well in mixed adversarial and stochastic models, and (ii) whose performance deteriorates gracefully as we move from fully stochastic to fully adversarial models.

In our model, the rewards for all arms are initially drawn from a distribution and are then altered by an adaptive adversary. We provide a simple algorithm whose performance gracefully degrades with the total corruption the adversary injected in the data, measured by the sum across rounds of the biggest alteration the adversary made in the data in that round; this total corruption is denoted by $C$. Our algorithm provides a guarantee that retains the optimal guarantee (up to a logarithmic term) if the input is stochastic and whose performance degrades linearly to the amount of corruption $C$, while crucially being agnostic to it. We also provide a lower bound showing that this linear degradation is necessary if the algorithm achieves optimal performance in the stochastic setting (the lower bound works even for a known amount of corruption, a special case in which our algorithm achieves optimal performance without the extra logarithm). 
\end{abstract}

\addtocounter{page}{-1}
\thispagestyle{empty}
\newpage

\section{Introduction}
\label{sec:intro}

In online learning with bandit feedback, a learner needs to decide at each time between alternative actions or \emph{arms} of unknown quality, facing a trade-off between \emph{exploiting} profitable past actions or \emph{exploring} new actions about which she has little information. Bandit problems are typically classified according to how the rewards are generated. In stochastic bandits, rewards are drawn from fixed but unknown distributions, which models settings where the alternatives follow particular patterns and do not react to the learner. The other extreme is adversarial bandits, which are robust to rewards that are specifically designed to trick the learner, as in game-theoretic settings.

In this paper, we focus on settings where the overall behavior is essentially stochastic but a small fraction of the rewards can be adversarially changed. Classic stochastic bandit algorithms, like Upper Confidence Bound (UCB) \cite{Auer2002} or {Active} Arm Elimination (AAE) \cite{Even-DarManMan06}, base most of their decisions on a few observations made in an initial phase of the algorithm and therefore can be easily tricked into incurring linear regret if very few arms are corrupted. Adversarial bandit algorithms like EXP3 are not fooled by such tricks, but cannot exploit the fact that the input is mostly stochastic.

Our goal is to robustify the stochastic setting by designing algorithms that can tolerate corruptions and still be able to exploit the stochastic nature of the input. The algorithms we design are \emph{agnostic} to the corruption, i.e. they can tolerate any level of corruption,
and the guarantee degrades gracefully as more corruption is added. Moreover, we prove lower bounds showing that our results are tight up to a logarithmic factor. Before we explain our technical contribution in detail, we describe examples of settings we have in mind.

\paragraph{Click fraud} In pay-per-click online advertising, the platform selects for each pageview an ad to display and obtains a certain reward if the user clicks on the ad. The click probabilities are unknown. The tension between repeatedly displaying a particular profitable ad that provides reliable revenue and exploring other potentially more rewarding options is a major application of stochastic bandits in the ads industry.

If it weren't for a phenomenon known as \emph{click fraud}, this would be a textbook example of stochastic bandits. In click fraud, botnets maliciously simulate users clicking on an ad to trick learning algorithms. One example is a bot consistently making searches to trigger some ad and not clicking on it to make it seem like a certain ad has very low click-through-rate in order to boost its competitor.

\paragraph{Recommendation systems:} A platform recommending activities or services to a user faces the same trade-off. Suggesting new restaurants leads to faster learning of the best spots but may result to dissatisfaction of the customers who are led to disappointing experiences. While most inputs follow a stochastic pattern, some inputs are typically corrupted: either maliciously, e.g. fake reviews by competitors, or non-maliciously, e.g. construction next-door makes the restaurant less desirable in certain interval. This corruption may again exhibit arbitrary patterns and is not identically distributed over time, yet it is dwarfed by the fact that most of the input is stochastic.

There are several other such examples: emails mostly follow a stochastic pattern except a fraction of them which are spam and are designed to trick algorithms. Internet searches follow a predictable pattern except certain spikes caused by unpredictable events. Data collection used in the econometric process often suffers from errors that affect a small part of the input. In all those cases, the vast majority of the input follows a predictable pattern, but a fraction of the samples are corrupted.

\subsection{Our contribution}

\textbf{Our model.} 
In this paper, we introduce a new model of stochastic bandits with adversarial corruptions. The goal of this model is to encourage the design of bandit algorithms that (i) work well in mixed adversarial and stochastic models, and (ii) whose performance deteriorates gracefully as we move from fully stochastic to fully adversarial models. 

In this model there are $K$ arms, each associated with a fixed reward distribution $\mathcal{F}(a)$. At each round $t$, a random reward $r^t_\s(a) \sim \mathcal{F}(a)$ is drawn and an adversary can change the reward to $r^t(a)$, possibly using information about the realizations of $r_\s^{\tau}(a)$ from both the current and previous rounds $\tau\leq t$ as well as the probability that the learner puts on each arm. The learner then draws an arm~$a^t$ and obtains $r^t(a^t)$ both as reward and feedback. We say that the adversary is $C$-corrupted if in every sample path we have $\sum_t \max_a \vert r^t(a) - r_\s^t(a)\vert \leq C$.

\textbf{Our results.}
The main result (Theorem \ref{thm:agnostic} in Section \ref{sec:upper_bound}) is a learning algorithm we term \textsf{Multi-layer Active Arm Elimination Race} that with probability $1-\delta$ has regret
$$\bigO\prn*{\sum_{a{\neq a^{\star}}}\frac{K\cdot C\log\prn*{\nicefrac{KT}{\delta}}+\log\prn*{T}}{\Delta(a)}\cdot\log\prn*{\nicefrac{KT}{\delta}}}$$
where $\Delta(a)$ is the \emph{gap} of arm $a$, i.e. the difference in stochastic means of arm $a$ and the optimal  arm~$a^{\star}$. 
For arms with very small gap, i.e. when $\Delta(a)\leq\nicefrac{1}{\sqrt{T}}$, the inverse dependence on the gap can be replaced by $\sqrt{T}$. It is possible to improve the bound by a log factor for pseudo-regret, i.e. maximum expected regret against any fixed arm, obtaining:
$\bigO\prn*{\sum_{a{\neq a^*}}\frac{K\cdot C+\log\prn*{T}}{\Delta(a)}\cdot\log(KT)}$.
Two important features of the algorithm are that the guarantee is:
\begin{itemize}
\item {\bf Agnostic}: The algorithm does not need to know the corruption level $C$. The guarantee is provided with respect to how much corruption was added in retrospect. If the corruption level is known, we can remove the dependence on $K\cdot\log(\nicefrac{KT}{\delta})$
as shown in Theorem \ref{thm:known_corruption}.

\item {\bf High Probability:} Our bounds hold with high probability which is important for practical applications as the ones described above. In contrast, the weaker definition of pseudo-regret often hides events with large regret that are offset by events with large negative regret.
\end{itemize}
The stochastic case corresponds to $C = 0$ in which case we recover a bound that is slightly worse than the guarantee provided by UCB. Our algorithm obtains $\bigO\prn*{\sum_{a\neq a^{\star}}\log(T) \cdot {\log(\nicefrac{T}{\delta}) /\Delta(a)}}$ with probability $1-\delta$, while UCB obtains this bound without the $\log(T)$ term.

En route to the result, in Theorem \ref{thm:stoch_known corruption} we show an algorithm that, for any fixed known $C$, provides regret $\bigO\prn*{\sum_{a\neq a^{\star}}\frac{\log\prn*{\nicefrac{KT}{\delta}}}{\Delta(a)}}$ for stochastic input and $\bigO\prn*{K\cdot C\cdot\sum_{a\neq a^{\star}}\frac{\prn*{\log\prn*{\nicefrac{KT}{\delta}}}^2}{\Delta(a)}}$ if it is $C$-corrupted. In other words, if we only need to tolerate either a known level $C$ or zero corruptions, we save a logarithmic factor from the bound, and match the bound provided by UCB in the stochastic case.

Another question is whether the linear dependence on the corruption level is tight. In Section \ref{sec:lower_bound}, we show that it cannot be improved upon without decay in the stochastic guarantee (i.e. while still guaranteeing logarithmic regret when the input is stochastic). The lower bound is an adaptation from the adversarial to the corrupted setting of a result from Auer and Chiang \cite{auer16}. This holds even for the case where the corruptions are either $0$ or a known level $C$ (where our algorithm provides a matching upper bound). We prove in Theorem \ref{thm:lower_bound_1} that an algorithm with pseudo-regret $\bigO\prn*{\nicefrac{\log(T)}{\Delta}}$ in the stochastic setting ($C=0$) then for every constant $\epsilon > 0$, there is a $O(T^\epsilon)$-corrupted instance where the algorithm incurs regret $\Omega(T^\epsilon)$ with constant probability.

Our algorithm can also be viewed through the lens of the best of both worlds literature \cite{BubeckS12,DBLP:conf/icml/SeldinS14,auer16,DBLP:conf/colt/SeldinL17}, where the goal is to design algorithms that simultaneously provide logarithmic regret guarantees in the stochastic regime and square-root guarantees in the adversarial. In Section~\ref{sec:extensions}, we sketch how our algorithm can be appropriately modified to obtain, for any constant $0<a<\nhalf$, $\widetilde\bigO\prn*{C}$ pseudo-regret for $C = O(T^a)$ and $\widetilde\bigO\prn*{T^{a+\nhalf}}$ pseudo-regret otherwise. We observe that the results in the best of both worlds literature correspond to the case where $a=0$. We note that such bounds are obtained for pseudo-regret and not regret with high-probability. 

\textbf{Our techniques.}
The starting point of our design are classical stochastic bandit learning algorithms
like UCB and {Active} Arm Elimination. Such algorithms are very susceptible to corruptions since they base most of their decisions on a small initial exploration phase. Therefore, with a small number of corruptions it is possible to completely trick the algorithm into eliminating the optimal arm.

We address this issue by robustifying them using a multi-layer approach. The learning algorithm consists of multiple layers running in parallel. The layers have decreasing speed and increasing tolerance to corruption. The first layer finishes very fast selecting an arm as optimal, but provides no tolerance to corruption. Subsequent layers are more robust but also slower.

The resulting algorithm is a \emph{race} between different layers for picking the optimal arm. Once the fastest layer finishes, it provides a first crude estimate of the optimal arm. Once slower layers finish, we obtain finer and finer estimates of the optimal arm. 

Our second main idea is that we can obtain more robust algorithms by \emph{subsampling}. If a layer is only selected with probability $p$, it only receives in expectation a $p$-fraction of the corruption injected by the adversary. If $p$ is low enough, the layer behaves almost as if it was stochastic.

Finally, we couple the different layers together by a process of \emph{global eliminations}. This process enables slower layers to eliminate arms in faster layers. Such a process is necessary for preventing inaccurate layers from pulling suboptimal arms too often.

\subsection{Related work}
\label{ssec:related_work}

Online learning with stochastic rewards goes back to the seminal work of Lai and Robbins \cite{Lai1985}. The case of adversarial rewards was introduced by Auer et al. \cite{AuerCeFrSc03}. The reader is referred to the books of Cesa-Bianchi and Lugosi \cite{prediction_book}, Bubeck and Cesa-Bianchi \cite{BubeckC12}, and Slivkins \cite{Slivkins} for an elaborate overview of the area.  These two extremes suffer from orthogonal problems; the one is overoptimistic expecting that all rewards come from the same distribution while the other one is too pessimistic in order to be protected against malicious adversaries. Our work addresses the middle ground: rewards come from distributions but are often adversarially corrupted. This is motivated by the non-robustness of stochastic learning algorithms to even small corruption levels. 

Closely related to our work lie  the works on best of both worlds guarantees \cite{BubeckS12,DBLP:conf/icml/SeldinS14,auer16,DBLP:conf/colt/SeldinL17}. These works achieve (up to logarithmic factors) the optimal pseudo-regret guarantee for stochastic rewards and the optimal pseudo-regret or actual regret guarantee for adversarial rewards. Bubeck and Slivkins \cite{BubeckS12} and Auer and Chiang \cite{auer16} begin from a stochastic algorithm and test whether they encounter non-stochastic behavior in which case they switch to adversarial algorithm. In contrast, Seldin et al. \cite{DBLP:conf/icml/SeldinS14,DBLP:conf/colt/SeldinL17} begin from an adversarial algorithm with very optimistic learning rate and adapt it if they encounter such behavior. Recently and independently to this work, Wei and Luo \cite{WeiLuo18} provide a best of both worlds result with a small-loss pseudo-regret guarantee on the adversarial setting, via a novel analysis of the log-barrier OMD algorithm of Foster et al. \cite{FosterLiLySrTa16}. Although the aforementioned algorithms are very elegant, their analysis is not robust to inputs that are slightly away from stochastic. Our work bridges this gap by designing algorithms with a more smooth behavior for close-to-stochastic instances.

There have been other works that attempt to provide improved guarantees than the adversarial setting when instances are well behaved. Hazan and Kale \cite{Hazan:2009:BAB:1496770.1496775} offer regret guarantees that scale with the variance of the losses instead of the time horizon. This guarantee is meaningful in settings that have a very predictable nature and have usually the same performance such as routing. However they do not address most applications of stochastic bandits. In Click Fraud, for example the rewards come from Bernoulli distributions and the variance of such a distribution is high even if the input is totally stochastic.  Another approach is the work of Shamir and Szlak \cite{DBLP:conf/icml/ShamirS17}, who consider an input that is adversarial but random local permutations are applied to obtain a more benign instance. This approach is very relevant in settings like buffering, but is again not applicable to our settings.

On the opposite side, attempting to provide improved guarantees for the stochastic setting or enhancing their range is a very active area of research. For instance, the MOSS algorithm \cite{audibert2009minimax} of Audibert and Bubeck provides the optimal non-distribution-based upper bound for stochastic bandits while retaining the optimal distribution-based stochastic guarantee. The KL-UCB algorithm of Garivier and Capp\'e \cite{GarivierCappe11} provides improved constants in the upper bound of the stochastic guarantee matching the lower bound of Lai and Robbins \cite{Lai1985} for Bernoulli rewards. The Robust UCB algorithm \cite{Bubeck2013bandits} extends the results to non-bounded rewards replacing with the weaker assumption of bounded variance. However, all the above results are not robust to corruptions from  an adaptive adversary due to their deterministic nature. Since the adversary knows the arm the learner will select, they can always corrupt the optimal arm whenever it is about to be selected and therefore cause the learner to either play it multiple times even if it is suboptimal or decide against playing it even with a small amount of corruption (similarly as in our lower bound).

There is also prior work on incorporating corruptions in online decision making. In the online learning front, there are two such attempts, to the best of our knowledge. In their best of both worlds result, Seldin and Slivkins \cite{DBLP:conf/icml/SeldinS14} allow for some contamination in the data as long as they are obliviously selected and they do not decrease the gap by more than a factor of $2$. The second work is a recent paper by Gajane et al. \cite{gajanekauffman3} who suggest a model of corrupted feedback aiming for differential privacy. Unlike our model, their corruptions are neither adversarial nor adaptive. Both of these works make benign assumptions about the nature of corruption and  do not address the main roadblock in the settings we consider: an adversarial saboteur will try to add faulty data in the beginning to change the order between the two arms and, with a minimal corruption, she will achieve this goal. Closer to our model are the works on robust allocation such as online matching with corrupted data \cite{Mirrokni:2012:SAA:2095116.2095250,DBLP:conf/sigecom/EsfandiariKM15}; unlike online matching though, in online learning we cannot evaluate the optimum at every round since the algorithm's decisions affect the information it observes.

Last, learning in the presence of corruptions has recently received great attention in the batch learning setting. For instance, recent works study inference under the presence of adversarially corrupted data \cite{Mansour:2015:RPI:2722129.2722160}, designing estimators that are robust to corrupted data \cite{DBLP:conf/focs/DiakonikolasKK016}, learning in auctions with some faulty data due to econometrics errors \cite{cai}. Our work suggests a similar framework for the study of online learning that is robust to adversarial corruptions in the more challenging problem of sequential decision making where decisions also affect the information observed. 

\section{Model}
\label{sec:model}

{\bf Corrupted stochastic bandits. } 
We study an online bandit learning setting with $K$ arms. Each arm $a \in [K]$ is associated with a distribution $\mathcal{F}(a)$ with mean $\mu(a)$. The distributions are assumed to have positive measure only on rewards in $[0,1]$ and are unknown to the learner. We refer to the optimal arm as $a^{\star}=\arg\max_a \mu(a)$ and define $\Delta(a)=\mu\prn*{a^{\star}}-\mu(a)$.\footnote{We note that $a^{\star}$ is one arm with optimal mean and this does not preclude the existence of other arms with the same mean. If more than one such arms exist, let $a^{\star}$ be an arbitrary arm with optimal mean and the other arms $a\neq a^{\star}$ with optimal mean have gap $\Delta(a)=0$.}

We consider an adversary who can corrupt some of the stochastic rewards. The adversary is adaptive, in the sense that the corrupted rewards can be a function of the realization of the stochastic rewards up to that point and of the learner's choices in previous rounds. More formally, the protocol between learner and adversary, at each round $t=1\dots T$, is as follows:
\begin{enumerate}
\item The learner picks a distribution $w^t$ over the $K$ arms.
\item Stochastic rewards are drawn for each arm: $r^t_\s(a)\sim \mathcal{F}(a)$.
\item The adversary observes the realizations of $r^t_\s(a)$ as well as rewards and choices of the learner in previous steps and returns a corrupted reward $r^t(a){\in[0,1]}$. 
\item The learner draws arm $a^t\sim w^t$ and observes $r^t\prn*{a^t}$.
\end{enumerate} 

We refer to $\max_a\abs{r^t(a)-r^t_\s(a)}$ as the amount of corruption injected in round $t$. The instance is $C$-corrupted if the total injected corruption is at most $C$ for all realizations of the random variables:
$$\sum_t \max_a \abs{r^t(a) - r^t_\s(a)} \leq C$$

Note that the adversary is assumed to be adaptive, in the sense that she has access to all the realizations of random variables for all rounds $\tau<t$ and the realization of rewards at round $t$ but only knows the player's distribution at round $t$ and not the arm $a^t$.
Our guarantees gracefully degrades with the total corruption injected by the adversary.

{\bf Regret notions.} Regret corresponds to the difference between the reward obtained by the algorithm and the reward of the best arm in hindsight:
$$
\regret= \max_a \sum_t  r^t\prn*{a} - r^t\prn*{a^t}
$$
The regret is a random variable that depends on the random rewards, the randomness used by the learner, and the randomness of the adversary. We say that a regret bound $R(T,\delta)$ holds with probability $1-\delta$ if $$\Pr\brk*{\regret < R(T,\delta)} > 1-\delta$$
where the probability is taken over all the three sources of randomness described. 

Finally pseudo-regret is a weaker notion that compares the expected performance of the learner with the arm with the highest expected performance. In other words:
$$
\pseudoregret=\max_a \En\brk*{ \sum_t r^t\prn*{a} - r^t\prn*{a^t}}
$$
We note that by Jensen's inequality, $\pseudoregret \leq \En\brk*{\regret}$. We often obtain improved bounds for pseudo-regret since it allows us to offset large positive regret events with large negative regret events.

\section{The upper bound: Multi-layer Active Arm Elimination Race}
\label{sec:upper_bound}

{\bf Active arm elimination.} The starting point of our design is the \emph{Active Arm Elimination} algorithm for stochastic bandits \cite{Even-DarManMan06}, which can be viewed as an alternative presentation of the more famous UCB algorithm \cite{Auer2002}. It is based on the following idea: in an initial \emph{exploration phase}, we pull arms in a round-robin fashion and compute an estimate $\widetilde\mu(a)$ as the average empirical reward of arm $a$. After $n(a)$ pulls of arm $a$, usual concentration arguments establish that with probability at least $1-\nicefrac{1}{T^{\Omega(1)}}$, the difference of the empirical and actual means is at most $\wid(a) = \bigO(\sqrt{\nicefrac{\log(T)}{n(a)}})$. We say that $[\widetilde\mu(a)-\wid(a),\widetilde\mu(a)+\wid(a) ]$ is the confidence interval of arm $a$.

This means in particular that given two arms $a$ and $a'$, if the difference in empirical means becomes larger than the widths of the confidence intervals, i.e., $\widetilde\mu(a) - \widetilde\mu(a') > \wid(a) + \wid(a')$, then with high probability arm $a'$ is not optimal. Once this happens, the algorithm eliminates arm $a'$ by removing it from the round-robin rotation. After both arms $a$ and the optimal arm $a^{\star}$ are pulled $\bigO(\nicefrac{\log(T)}{\Delta(a)^2})$ times, the confidence intervals will be small enough that arm $a$ will be eliminated.

Eventually all arms but the optimal are eliminated and we enter what is called the \emph{exploitation phase}. In this phase we only pull the arm with optimal mean. Before we enter exploitation we pulled each suboptimal arm $a$ at most $\bigO(\nicefrac{\log(T)}{\Delta(a)^2})$ times. Each of those suboptimal pulls incurs regret $\Delta(a)$ in expectation which leads to the pseudo-regret bound of $\bigO(\sum_{a \neq a^\star} \nicefrac{\log(T)}{\Delta(a)})$. This bound can also be converted to a high probability bound if we replace $\log(T)$ by $\log(\nicefrac{T}{\delta})$. 

{\bf Arms with small $\Delta(a)$.} We note that, for the arms that have $\Delta(a)<\nicefrac{1}{\sqrt{T}}$, the inverse dependence on the gap may initially seem vacuous; for instance, when there are two optimal arms $a,a^{\star}$ with the same mean, the upper bound becomes infinite as $\Delta(a)=0$. However, the inverse dependence on the gap can be replaced by $\Delta(a)\cdot T$ in the case of pseudo-regret and $\sqrt{T}$ in the case of actual regret (due to variance reasons). For simplicity of exposition, we omit this in the current section
but we demonstrate how to perform this replacement in Section \ref{sec:extensions}.

\subsection{Enlarged confidence intervals}\label{ssec:known}
The active arm elimination algorithm is clearly not robust to corruption since by corrupting the first $\bigO(\log T)$ steps, the adversary can cause the algorithm to eliminate the optimal arm. As the algorithm never pulls the suboptimal arms after exploration, it is not able to ever recover. One initial idea to fix this problem is to enlarge the confidence intervals. We can decompose the rewards $r^t(a)$ in two terms $r^t_\s(a) + c^t(a)$ where the first term comes from the stochastic reward and the second is the corruption introduced by the adversary. If the total corruption introduced by the adversary is at most $C$, then with width $\wid(a) = \bigO(\sqrt{\nicefrac{\log(T)}{n(a)}} + \nicefrac{C}{n(a)})$, a similar analysis to above gives us the following regret bound:
\begin{theorem}\label{thm:known_corruption}
If $C$ is a valid upper bound for the total corruption then active arm elimination with $\wid(a) = \sqrt{\frac{\log\prn*{\nicefrac{2KT}{\delta}}}{n(a)}}+\frac{C}{n(a)}$ has regret $\bigO\prn*{\sum_{a \neq a^\star}\prn*{\frac{\log\prn*{\nicefrac{KT}{\delta}}+C}{\Delta(a)}}}$ with probability $1-\delta$.
\end{theorem}
\begin{proof}[Proof sketch]
The proof follows the standard analysis of active arm elimination. We first establish that, with high probability the optimal arm $a^{\star}$ is never inactivated (Lemma \ref{lem:optimal_survives}) and then upper bound the number of times each suboptimal arm is played (Lemma \ref{lem:bound_suboptimal_plays}). The pseudo-regret guarantee directly follows by multiplying the number of plays for each arm by its gap $\Delta(a)$. For the high-probability guarantee, we need to also show that the regret incurred in the meantime is not much more than the above. We provide proof details about the theorem and lemmas in Appendix~\ref{app:sec_known}.
\end{proof}
\begin{lemma}\label{lem:optimal_survives}
With probability at least $1-\delta$, arm $a^{\star}$ never becomes inactivated.
\end{lemma}
\begin{lemma}\label{lem:bound_suboptimal_plays}
With probability at least $1-\delta$, all arms $a\neq a^{\star}$ become inactivated after $N(a)=\frac{
36\log\prn*{\nicefrac{2KT}{\delta}}+
6 C}{\Delta(a)^2}$ plays.
\end{lemma}

\subsection{Stochastic bandits robust to known corruption}
\label{ssec:robust_known}

The drawback of the active arm elimination algorithm with enlarged confidence intervals (Theorem~\ref{thm:known_corruption}) is that, even if there are no corruptions, it still incurs a regret proportional to $C$. As a warm up to the main theorem, we provide an algorithm that achieves the usual bound of  $\bigO\prn*{\sum_{a \neq a^{\star}}\frac{\log{\prn*{\nicefrac{KT}{\delta}}}}{\Delta(a)}}$ if the input is purely stochastic and,  at the same time, achieves $\bigO\prn*{K\cdot C\cdot\sum_{a \neq a^{\star}}\frac{\log\prn*{\nicefrac{KT}{\delta}}^2}{\Delta(a)}}$ if the input is $C$-corrupted for a known $C$.
In the next subsection, we modify the algorithm to make it agnostic to the corruption level $C$.

\textbf{Two instances of Active Arm Elimination.} The first idea is to run two instances of active arm elimination: the first is supposed to select the correct arm if there is no corruption and the second is supposed to select the right arm if there is $C$ corruption. The first instance is very fast but it is not robust to corruptions. The second instance is slower but more precise, in the sense that it can tolerate corruptions. Since the second instance is more trustworthy, if the second instance decides to eliminate a certain arm $a$, we eliminate the same arm in the faster instance.

\newcommand{\F}{\mathsf{F}}
\renewcommand{\S}{\mathsf{S}}

\textbf{Decrease corruption by sub-sampling.} To keep the regret low if the input is stochastic, the second instance of active arm elimination cannot pull a suboptimal arm too many times. Therefore, the technique in Theorem~\ref{thm:known_corruption} alone is not enough. The \emph{main idea} of the algorithm is to make arm~$a$ behave as if it was almost stochastic by running the second instance with low probability. If the learner selects to run the second instance with probability $\nicefrac{1}{C}$ then, when the adversary adds a certain amount of corruption to a certain round, the second instance observes that corruption with probability $\nicefrac{1}{C}$. Therefore, the expected amount of corruption the learner observes in the second instance is constant. This makes the arms behave almost like stochastic arms in that instance.

\textbf{Learning algorithm.} We obtain our algorithm by combining those ideas. We have two instances of active arm elimination which we denote by $\F$ (fast) and $\S$ (slow). Each instance keeps an estimate of the mean $\widetilde{\mu}^\F(a)$ and $\widetilde{\mu}^\S(a)$  corresponding to the average empirical reward of that arm and also keeps track of how many times each arm was pulled in that instance $n^\F(a)$ and $n^\S(a)$. This allows us to define a notion of confidence interval in each of the instances. We define $\wid^\F(a) = \bigO(\sqrt{\nicefrac{\log(T)}{n^\F(a)}})$ as usual and for the slow instance we define slighly larger confidence intervals: $\wid^\S(a) = \bigO(\sqrt{\nicefrac{\log(T)}{n^\S(a)}} + \nicefrac{\log(T)
}{n^\S(a)})$ (the reason will be clear in a moment). Also, each instance keeps a set of eliminated arms for that instance: $\mathcal{I}^\F$ and $\mathcal{I}^\S$.

In each round, with probability $1-\nicefrac{1}{C}$ we make a move in the fast instance: we choose the next active arm $a$ in the round robin order, i.e., arm  $a \in \brk*{K}\setminus \mathcal{I}^\F$ which was played less often, pull this arm and increase $n^\F(a)$ and update $\widetilde{\mu}^\F(a)$ accordingly. As usual, if there are two active arms $a$ and $a'$ such that $\widetilde{\mu}^\F(a) - \widetilde{\mu}^\F(a') > \wid^\F(a) + \wid^F(a')$ we eliminate $a'$ by adding it to $\mathcal{I}^\F$.

With the remaining probability we make a move in the slow instance by executing the exact same procedure as described for the other instance. There is only one difference (which causes the two instances to be coupled): when we inactivate an arm $a$ in $\S$ we also eliminate it in $\F$. This leaves us with a potential problem: it is possible that all arms in the $\F$ instance end up being eliminated. If we reach that point, we play an arbitrary active arm of the slow instance, i.e., any arm $a \in \brk*{K} \setminus \mathcal{I}^\S$.


The resulting algorithm is formally provided in Algorithm \ref{alg:known}.
\begin{algorithm}
\begin{algorithmic}[1]
\caption{Fast-Slow Active Arm Elimination Race for known corruption $C$}
\label{alg:known}
\STATE Initialize $n^\ell(a)=0$, $\widetilde{\mu}^\ell(a)=0$, $\mathcal{I}^\ell = \emptyset$ for all $a \in \brk*{K}$ and $\ell \in \crl*{\F,\S}$
\STATE {\bf For} Rounds $t=1..T$
\STATE $\quad$ Sample algorithm $\ell$: $\ell=\S$ with probability $1/C$. Else $\ell=\F$.
\STATE $\quad$ {\bf If} $\brk*{K}\setminus\mathcal{I}^\ell \neq \emptyset$
\STATE $\quad$ $\quad$ Play arm $a^t \leftarrow \argmin_{a \in\brk*{K}\setminus\mathcal{I}^\ell} n^\ell(a) $
\STATE $\quad$ $\quad$ Update $\widetilde{\mu}^\ell(a^t) \leftarrow [n^\ell(a) \widetilde{\mu}^\ell(a^t) + r^t(a^t) ] / [n^\ell(a) + 1]$ and $n^
\ell(a) \leftarrow n^\ell(a) + 1 $
\STATE $\quad$ $\quad$ {\bf While} exists arms $a,a' \in \brk*{K}\setminus\mathcal{I}^\ell$ with $ \widetilde{\mu}^\ell(a) -\widetilde{\mu}^\ell(a') > \wid^\ell(a) + \wid^\ell(a') $
\STATE $\quad$ $\quad$ $\quad$ Eliminate $a'$ by adding it to $\mathcal{I}^{\ell}$
\STATE 
$\quad$
$\quad$
$\quad$
{\bf If} $\ell=\S$
{\bf then} eliminate $a'$ from the other algorithm by adding it to $\mathcal{I}^{\F}$
\STATE
$\quad$ {\bf Else}
\STATE $\quad$ $\quad$ Play an arbitrary arm in the set $\brk*{K}\setminus \mathcal{I}^{\S}$.
\end{algorithmic}
\end{algorithm} 

Towards the performance guarantee, Lemma \ref{lem:actual_corruption_high_prob} bounds the amount of corruption that actually enters the slow active arm elimination algorithm, which enables the regret guarantee in Theorem \ref{thm:stoch_known corruption}.
\begin{lemma}\label{lem:actual_corruption_high_prob}
In Algorithm \ref{alg:known}, the slow active arm elimination algorithm $\S$ observes, with probability at least $1-\delta$, corruption of at most $\ln(\nicefrac{1}{\delta})+3$ during its exploration phase (when picked with probability $\nicefrac{1}{C}$). \end{lemma}
\begin{proof}[Proof sketch]
If one cared just about the expected corruption that affects $\S$, this is at most a constant number since the total corruption is at most $C$ and it affects $\S$ with probability $\nicefrac{1}{C}$. To prove a high-probability guarantee we require a concentration inequality on martingale differences (since the corruptions can be adaptively selected by the adversary). We provide the details in Appendix \ref{app:sec_robust_known}.
\end{proof}

\begin{theorem}\label{thm:stoch_known corruption}
Algorithm \ref{alg:known} run with widths
$\wid^{\S}(a)=\sqrt{\frac{\log(\nicefrac{8KT}{\delta})}{n^{\S}(a)}}+\frac{2\log(\nicefrac{8KT}{\delta})}{n^{\S}(a)}$ and  $\wid^{\F}(a)=\sqrt{\frac{\log(\nicefrac{8KT}{\delta})}{n^{\F}(a)}}$ has $\bigO\prn*{\sum_{a\neq a^{\star}}\frac{\log\prn*{\nicefrac{KT}{\delta}}}{\Delta(a)}}$  for the stochastic case and $\bigO\prn*{K\cdot C\cdot\sum_{a\neq a^{\star}}\frac{\prn*{\log\prn*{\nicefrac{KT}{\delta}}}^2}{\Delta(a)}}$ for the $C$-corrupted case with probability at least $1-\delta$.
\end{theorem}

\begin{proof}[Proof sketch]
The result for the stochastic case follows standard arguments for stochastic algorithms (since we obtain double the regret of this setting as we run two such algorithms with essentially the same confidence intervals). For the $C$-corrupted case, we establish via Lemma \ref{lem:actual_corruption_high_prob} an upper bound on the corruption that will affect the slow active arm elimination algorithm $\S$. Thanks to the sub-sampling, this upper bound is close to a constant instead of depending on $C$ which allows to not incur dependence on $C$ in the stochastic case. Having this upper bound, we can apply it to the algorithm of the previous section to get an upper bound on the number of plays of suboptimal arms in $\S$. Since the algorithms are coupled, such a bound implies an upper bound on the regret that it can cause in $\F$ as well. This is because in expectation the arm is played at most $K\cdot C$ times more in $\F$ as it may be selected every single time in $\F$ prior to getting eliminated by $\S$ and $\F$ is selected $C$ times more often than $\S$. To obtain the above guarantee with high probability, we lose an extra logarithmic factor. The details of the proof are provided in Appendix \ref{app:sec_robust_known}.
\end{proof}

\subsection{Stochastic bandits robust to agnostic corruption}
\label{ssec:robust_agnostic}

\textbf{Multiple layers of active arm elimination.} In the previous subsection we designed an algorithm with two layers: one is faster but cannot tolerate corruptions and the second one is slower but more robust. In order to be agnostic to corruption, we need to plan for all possible amounts of corruption. To achieve this, we introduce $\log(T)$ layers. Each layer is slower but more robust than the previous one. We achieve that by selecting the $\ell$-th layer with probability proportional to $2^{-\ell}$. By the argument in the last section, if the corruption level is at most $C$, then each layer with $\ell \geq \log C$ will observe $\bigO(1)$ corruption in expectation and at most $\bigO(\log T)$ corruption with high probability.

\textbf{Global eliminations.} We couple the $\log T$ instances through what we call global eliminations. If arm $a$ is eliminated by the $\ell$-th layer, then we eliminate $a$ in all layers $\ell' \leq \ell$. This is important to prevent us from pulling arm $a$ too often. If arm $a$ is suboptimal and the adversary is $C$-corrupted, then arm $a$ eventually becomes eliminated in the $\ell^{\star} = \lceil\log C\rceil$ layer after being pulled $\widetilde\bigO(\nicefrac{1}{\Delta(a)^2})$ in that layer. Since layer $\ell^{\star}$ is played with probability $2^{-\ell^{\star}}$ then it takes $\widetilde\bigO( \nicefrac{C}{\Delta(a)^2})$ iterations until arm is eliminated globally, in which case we will have total regret at most $\widetilde\bigO( \nicefrac{C}{\Delta(a)})$ from that arm.

\textbf{Multi-layer active arm elimination race.} We now describe our main algorithm in the paper. We call it a \emph{race} since we view it as multiple layers racing to pick the optimal arm. The less robust layers are faster so they arrive first and we keep choosing (mostly) according to them until more robust but slower layers finish and correct or confirm the current selection of the best arm.

The algorithm keeps $\ell=1\dots\log(T)$ different instances of active arm elimination. The $\ell$-th instance has as state the empirical means of each arm $\widetilde{\mu}^\ell(a)$, the number $n^\ell(a)$ of times each arm $a$ was pulled and the set $\mathcal{I}^\ell$ of inactive arms. The width of the confidence interval for arm $a$ in the $\ell$-th layer is implicitly defined as $\wid^\ell(a) = \bigO(\sqrt{\nicefrac{\log(T)}{n^\ell(a)}} + \nicefrac{\log(T)}{n^\ell(a)})$.

In each round $t$ we sample $\ell \in \{1,\hdots, \log T\}$ with probability $2^{-\ell}$ (with remaining probability we pick layer $1$). When layer $\ell$ is selected, we make a move in the active arm elimination instance corresponding to that layer: we sample the active arm in that layer with the least number of pulls, i.e., arm $a \in \brk*{K}\setminus \mathcal{I}^\ell$ minimizing $n^\ell(a)$. In case $\brk*{K}\setminus \mathcal{I}^\ell$ is empty, we pull an arbitrary arm from $\brk*{K}\setminus \mathcal{I}^{\ell'}$ for the lowest $\ell'$ such that $\brk*{K}\setminus \mathcal{I}^{\ell'}$ is non-empty.

The way we couple different layers is that once arm $a'$ is eliminated in layer $\ell$ because there is another active arm $a$ in layer $\ell$ such that $\widetilde{\mu}^\ell(a) - \widetilde{\mu}^\ell(a') < \wid^\ell(a) + \wid^\ell(a')$ we eliminate arm $a'$ in all previous layers, keeping the invariant that:
$\mathcal{I}^1 \supseteq \mathcal{I}^2 \supseteq \mathcal{I}^3 \supseteq \hdots$.

Figure \ref{fig:algo_state} provides an example of the state of the algorithm, which is formally defined in Algorithm \ref{alg:agnostic}. 
\label{app:figure}
\begin{figure}[h]
\centering
  \begin{tikzpicture}
  \fill[red!40!white] (0,5-0)--(0,5-1)--(3,5-1)--(3,5-3)--(6,5-3)--(6,5-1)--(9,5-1)--(9,5-5)--(12,5-5)--(12,5-0)--cycle;
  \begin{scope}[xscale=3]
      \draw[step=1cm,black] (0,0) grid (4,5);
  \end{scope}
  \foreach \y in {1, 2, 3,5} {
      \foreach \x in {1, 2, 4} {
        \node at (3*\x-1.5,5.5-\y) {$\widetilde\mu^{\ifthenelse{\y=5}{\lg T}{\y}}(\ifthenelse{\x=4}{d}{\x}), n^{\ifthenelse{\y=5}{\lg T}{\y}}(\ifthenelse{\x=4}{d}{\x})$};
        \node at (3*3-1.5,5.5-\y) {$\hdots$};
      }
      \foreach \x in {1, 2} {
        \node at (3*\x-1.5,5.2) {arm $\x$};
      }
      \node at (3*4-1.5,5.2) {arm $d$};
      \node at (3*3-1.5,5.2) {$\hdots$};
      \node at (-1,5.5-\y) {$\ell = \ifthenelse{\y=5}{\lg T}{\y}$};
  }
  \foreach \x in {1, 2, 4} {
    \node at (3*\x-1.5,5.5-4) {$\vdots$};
  }
  \node at (-1,5.5-4) {$\vdots$};

  \end{tikzpicture}
  \caption{Example of the state of the algorithm: for each layer $\ell$ and arm $a$ we keep the estimated mean $\widetilde\mu^\ell(a)$ and the number of pulls $n^\ell(a)$. Red cells indicate arms that have been eliminated in that layer. If an arm is eliminated in a layer, it is eliminated in all previous layers. If a layer where all the arms are eliminated (like layer $1$ in the figure) is selected, we play an arbitrary active arm with the lowest layer that contains active arms.}\label{fig:algo_state}
\end{figure}

\begin{algorithm}
\begin{algorithmic}[1]
\caption{Multi-layer Active Arm Elimination Race}
\label{alg:agnostic}
\STATE Initialize $n^\ell(a)=0$, $\widetilde{\mu}^\ell(a)=0$, $\mathcal{I}^\ell = \emptyset$ for all $a \in \brk*{K}$ and $\ell \in [\log T]$
\STATE {\bf For} Rounds $t=1..T$
\STATE $\quad$ Sample layer $\ell \in [\log T]$ with probability $2^{-\ell}$. With remaining prob, sample $\ell = 1$
\STATE $\quad$ {\bf If} $\brk*{K}\setminus\mathcal{I}^\ell \neq \emptyset$
\STATE $\quad$ $\quad$ Play arm $a^t \leftarrow \argmin_{a \in\brk*{K}\setminus\mathcal{I}^\ell} n^\ell(a) $
\STATE $\quad$ $\quad$ Update $\widetilde{\mu}^\ell(a^t) \leftarrow [n^\ell(a) \widetilde{\mu}^\ell(a^t) + r^t(a^t) ] / [n^\ell(a) + 1]$ and $n^
\ell(a) \leftarrow n^\ell(a) + 1 $
\STATE $\quad$ $\quad$ {\bf While} exists arms $a,a' \in \brk*{K}\setminus\mathcal{I}^\ell$ with $ \widetilde{\mu}^\ell(a) -\widetilde{\mu}^\ell(a') > \wid^\ell(a) + \wid^\ell(a') $
\STATE $\quad$ $\quad$ $\quad$ Eliminate $a'$ by adding it to $\mathcal{I}^{\ell'}$ for all $\ell' \leq \ell$
\STATE $\quad$ {\bf Else}
\STATE $\quad$ $\quad$ Find minimum $\ell'$ such that $\brk*{K}\setminus\mathcal{I}^{\ell'} \neq \emptyset$ and play an arbitrary arm in that set.
\end{algorithmic}
\end{algorithm} 

We now provide the main result of the paper, a regret guarantee for Algorithm \ref{alg:agnostic}.


\begin{theorem}\label{thm:agnostic}
Algorithm \ref{alg:agnostic} which is agnostic to the coruption level $C$, when run with widths
$\wid^{\ell}(a)=\sqrt{\frac{\log(\nicefrac{4KT\cdot \log T}{\delta})}{n^{\ell}(a)}}+\frac{\log(\nicefrac{4KT\cdot \log T}{\delta})}{n^{\ell}(a)}$ has regret: $$\bigO\prn*{\sum_{a\neq a^*}\frac{K\cdot C\log\prn*{\nicefrac{KT}{\delta}}+\log\prn*{T}}{\Delta(a)}\cdot\log\prn*{\nicefrac{KT}{\delta}}}.$$
\end{theorem}
\begin{proof}[Proof sketch]
Similarly to the previous theorem, the regret guarantee comes from the summation between layers that are essentially stochastic (where the corruption is below their corruption level, i.e. less than $C\leq 2^r$ for layer $r$). From each of these layers, we incur $\bigO\prn*{\frac{\log\prn*{\nicefrac{KT}{\delta}}}{\Delta(a)}}$ regret. Since there are at most $\log(T)$ such layers, the second term in the theorem is derived.
The challenge is to bound the regret incurred by layers that are not robust to the corruption. However, there exists some layer $\ell^{\star}$ that is above the corruption level. By bounding the amount of steps that this level will require in order to inactivate each arm $a\neq a^{\star}$ in the incorrect layers (via Lemma \ref{lem:bound_suboptimal_plays}), we obtain similarly to Theorem \ref{thm:stoch_known corruption} a bound on the regret caused by this arm in those layers.  Since we take the minimum such layer and the tolerance of layers is within powers of $2$, the fact that its corruption level does not match exactly the corruption that occurred only costs an extra factor of $2$ in the regret. The details of the proof are provided in Appendix \ref{app:sec_robust_agnostic}. 
\end{proof}

\section{The lower bound}
\label{sec:lower_bound}
For the two arms case where the gap between the arms is $\Delta > 0$,
Theorem \ref{thm:stoch_known corruption} presents an algorithm which achieves
$\bigO(\nicefrac{\log T}{\Delta})$ pseudo-regret if the input is stochastic and 
$\bigO(C \log (\nicefrac{T}{\delta}) /\Delta)$ with probability $1-\delta$ if
the input is at most $C$-corrupted. We show below that this dependence is
tight.

The lower bound (Theorem \ref{thm:lower_bound_1}) adapts the technique of 
Auer and Chiang \cite{auer16} from the adversarial to the corrupted setting. 
The main idea is that an algorithm with logarithmic regret in the stochastic
setting cannot query the sub-optimal arm more than $\nicefrac{\log(T)}{\Delta^2}$ times.
This implies a long time period where the learner queries the
input only constant number of times. By corrupting all rounds
before this period, an adversary can make the optimal arm look sub-optimal and trick the
learner into not pulling the optimal arm for long time, causing large
regret. Theorem \ref{thm:lower_bound_2} adapts this argument bounding the expected positive regret
$\E[\regret^+]$ where $x^+ = \max\{x,0\}$; the high probability bounds provided
imply bounds on the expected positive regret. Both proofs are provided in Appendix~\ref{app:sec_lower_bound}.

\begin{theorem}\label{thm:lower_bound_1} Consider a multi-armed bandits algorithm that has the property
that for any stochastic input in the two arm setting, it has pseudo-regret
bounded by $c \log(T) / \Delta$,
where $\Delta = \abs{\mu_1 - \mu_2}$. For any $\epsilon, \epsilon' \in (0,1)$,
there is a corruption level $C$ with $T^\epsilon < C < T^{\epsilon'}$ and a
$C$-corrupted instance such that with constant probability the regret is
$\Omega(C)$. 
\end{theorem}

\begin{theorem}\label{thm:lower_bound_2} If a multi-armed bandits algorithm that has the property
that for any stochastic input in the two arm setting, it has pseudo-regret
bounded by $c \log^{1+\alpha}(T) / \Delta$ for $\alpha < 1$.
For any $\epsilon, \epsilon' \in (0,1)$,
there is a corruption level $C$ with $T^\epsilon < C < T^{\epsilon'}$ and a
$C$-corrupted instance such that
$\E[\regret^+] = \Omega(T^{\epsilon - \delta})$ for all $\delta > 0$.
\end{theorem}

\section{Extensions}
\label{sec:extensions}

In this section, we discuss some extensions that our algorithm can accommodate.

\textbf{Definition of corruption.} We presented all results measuring the corruption as the sum over all rounds of the maximum across arms of the corruption injected by the adversary: $$\sum_t \max_a |r^t(a)-r_{\s}^t(a)\leq C.$$ In fact all our results can be improved via using $C(a)=\sum_t |r^t(a)-r_{\s}^t(a)|$ and replacing $C$ by $\max\prn*{C(a),C(a^{\star})}$ for summand $a$. More formally, our main theorem (Theorem \ref{thm:agnostic}) becomes:
\begin{theorem}\label{thm:agnostic_per_arm_corruption}
Algorithm \ref{alg:agnostic} which is agnostic to the corruptions $C(a)=\sum_t |r^t(a)-r_{\s}^t(a)|$, when run with widths
$\wid^{\ell}(a)=\sqrt{\frac{\log(\nicefrac{4KT\cdot \log T}{\delta})}{n^{\ell}(a)}}+\frac{\log(\nicefrac{4KT\cdot \log T}{\delta})}{n^{\ell}(a)}$ has regret: $$\bigO\prn*{\sum_{a\neq a^*}\frac{K\cdot \max\prn*{C(a^{\star}),C(a)}\cdot\log\prn*{\nicefrac{KT}{\delta}}+\log\prn*{T}}{\Delta(a)}\cdot\log\prn*{\nicefrac{KT}{\delta}}}.$$
\end{theorem}
The proof follows the same arguments  since it only compares each arm $a$ with $a^{\star}$. This result is nice since the contribution of each arm to the regret is a function only of its own gap and the corruption injected to it and the one injected to arm $a^{\star}$. The latter dependence on the corruption on the optimal arm is essential since the main attack we presented to the classical arguments only corrupts arm $a^{\star}$ -- the lower bound of the previous section also only adds corruption to $a^{\star}$.

\textbf{Dependence on the gap.} 
In Section \ref{sec:upper_bound}, all our guarantees have an inverse dependence on the gap $\Delta(a)$ of all arms $a$. Note that such a guarantee is completely meaningless for arms with a very small gap; for instance, if there exist two optimal arms then there is an arm $a\neq a^{\star}$ with $\Delta(a)=0$ which makes the presented bound infinite and therefore vacuous. As we hinted there though, this inverse dependence can be improved for arms with small $\Delta(a)\leq \nicefrac{1}{\sqrt{T}}$. Our proofs generally relied on setting an upper bound on the number of times that a suboptimal arm is played and thereby providing an upper bound on the regret they cause. 

For arms with $\Delta(a)\leq \nicefrac{1}{\sqrt{T}}$, an alternative analysis is to say that, even if they are erroneously selected every single time, we can upper bound the loss in performance they cause. For pseudo-regret, the performance loss if they were selected every single time is $\Delta(a)\cdot T\leq \sqrt{T}\leq \nicefrac{\log T}{\Delta(a)}$. For actual regret, one needs to also take into consideration the variance but, even if they are selected every single time, a Hoeffding bound shows that their total reward is with high probability at most $\sqrt{T}$ lower than its expectation. As a result, the inverse dependence on $\Delta(a)$ in our bound can be replaced by $\min(\Delta(a)\cdot T,\nicefrac{1}{\Delta(a)})$ for pseudo-regret and $\min(\sqrt{T},\nicefrac{1}{\Delta(a)})$ for actual regret.

Moreover, the careful reader may have noticed that in Theorem \ref{thm:known_corruption}, the dependence $\frac{C}{\Delta(a)}$ can be replaced by a sole dependence on $C$ without the gap. However, this does not extend to the subsequent theorems since the dependence on $C$ there does not come from the upper bound on the corruption experienced (this is at most $\log T$ due to subsampling). Instead, the dependence on $C$ comes from projecting the correct layer (smallest layer robust to corruption) to the previous layers via the number of times it will take to eliminate any suboptimal arm.

\textbf{Uncorrupted objective.} In applications such as spam, the corruptions should not be counted as part of the rewards. Our algorithm provides the same guarantee in the case of uncorrupted rewards (the difference between the performances in the two objectives is at most $C$). One can also observe that the linear dependence on $C$ is still necessary: consider $2$ arms with $\Delta=1$ and an adversary that corrupts the first $C$ steps making them look identical. The learner has no better option than randomly selecting between the two which gives him a regret of $\nicefrac{C}{2}$ under the uncorrupted objective. We note that, in this setting, the linear dependence is necessary unconditionally of the performance of the algorithm in the stochastic setting.

\textbf{Towards best of \emph{all} worlds.} In the previous section, we showed that a logarithmic dependence in the stochastic setting comes at the expense of linear dependence on $C$ in the $C$-corrupted setting if we focus on actual regret. A very interesting direction is to achieve such an improvement with either a higher power on the logarithm in the stochastic setting or aiming for pseudo-regret instead. 

In fact, we can combine our algorithm with the SAPO algorithm of Auer and Chiang \cite{auer16} and achieve a bicriteria guarantee for pseudo-regret. For an $a<\nhalf$ specified by the algorithm, we achieve our guarantee if the corruption is $C\leq T^a$ and at most $T^{\nhalf+a}$ otherwise; notice that the case $a=0$ corresponds to the best of both worlds. This is done via running the SAPO algorithm at the level $a\log(T)$ with probability $T^{-a}$ instead of having higher layers. The SAPO algorithm guarantees that the pseudo-regret caused by any particular arm is at most logarithmic if the instance is stochastic and at most $\sqrt{T}$ if it is adversarial via a beautiful analysis that keeps negative regret of time intervals that have performed well to avoid testing eliminated arms too often. In our setting, if the corruption level is less than $T^a$, the instance behaves as stochastic causing at most logarithmic regret. Else the instance is corrupted and we can extrapolate the regret in this layer to the whole algorithm as arms that are eliminated in this layer are also eliminated before via global eliminations. Since the regret there is at most $\sqrt{T}$ and this is multiplied by $T^a$, this implies a bound of $T^{\nhalf+a}$ on pseudo-regret.

\textbf{Acknowledgements} The authors would like to thank Sid Banerjee whose lecture notes on stochastic bandits proved very helpful, Andr\'es Munoz Medina, Karthik Sridharan, and \'Eva Tardos for useful discussions, Manish Raghavan for suggestions on the writeup, and the anonymous reviewers for the valuable feedback they provided that improved the presentation of the paper.
\bibliographystyle{alpha}
\bibliography{bib1}

\appendix

\section{Supplementary material on Section \ref{ssec:known}}
\label{app:sec_known}
In this section we provide the proof of Theorem \ref{thm:known_corruption}. Note that in the lemma statements the width is defined as in the theorem: $\wid(a) = \sqrt{\frac{\log\prn*{\nicefrac{2KT}{\delta}}}{n(a)}}+\frac{C}{n(a)}$ for any arm $a\neq a^{\star}$.

\textbf{Lemma \ref{lem:optimal_survives} (restated) }
With probability at least $1-\delta$, arm $a^{\star}$ never becomes eliminated.
\begin{proof} 
The crux of the proof lies in establishing that, with high probability, the upper bound of the confidence interval of $a^{\star}$ never becomes lower than the lower bound of the confidence interval of any other arm $a$ and therefore $a^{\star}$ does not become eliminated.

More formally, let
$\widetilde{\mu}_{\s}(a)$ and $\widetilde{\mu}(a)$ be the empirical mean after $n(a)$ samples of the stochastic part of the rewards and the empirical mean after $n(a)$ samples of the corrupted rewards respectively. Recall that $\mu(a)$ is the mean of arm $a$. By Hoeffding inequality, for any arm $a$, with probability at least $1-\delta'$: 
\begin{equation}\label{eq:main_hoeffding_known}\abs{\widetilde{\mu}_{\s}(a)-\mu(a)}\leq \sqrt{\frac{\log\prn*{\nicefrac{2}{\delta'}}}{n(a)}}.\end{equation}
We set $\delta'=\nicefrac{\delta}{KT}$ to establish that this holds for all arms and all time steps (after arm $a$ has been played $n(a)$ times). As a result, for any arm $a$ and any time: $\widetilde{\mu}_{\s}\prn*{a}\leq \mu(a)+\sqrt{\frac{\log(\nicefrac{2KT}{\delta})}{n(a)}}$ and $\widetilde{\mu}_{\s}\prn*{a^{\star}}\geq \mu(a^{\star})-\sqrt{\frac{\log(\nicefrac{2KT}{\delta})}{n(a^{\star})}}$. 

Comparing now the actual (corrupted) empirical means, they can be altered by at most absolute corruption $C$. Hence $\widetilde{\mu}(a)\leq \widetilde{\mu}_{\s}(a)+\frac{C}{n(a)}$ and $\widetilde{\mu}(a^{\star})\geq \widetilde{\mu}_{\s}(a^{\star})-\frac{C}{n(a^{\star})}$. 

Combining the above inequalities with the fact that the actual mean of $a^{\star}$ is higher than the one of $a$, i.e. $\mu(a^{\star})\geq \mu(a)$, we establish that $\widetilde{\mu}\prn*{a}-\widetilde{\mu}\prn*{a^{\star}}\leq \wid(a)+\wid(a^{\star})$ and therefore arm $a^{\star}$ is not eliminated. Since this holds for all times and arms, the lemma follows.
\end{proof}

\textbf{Lemma \ref{lem:bound_suboptimal_plays} (restated)}
With probability at least $1-\delta$, all arms $a\neq a^{\star}$ become eliminated after $N(a)=\frac{36\cdot\log\prn*{\nicefrac{2KT}{\delta}}+6C}{\Delta(a)^2}$ plays. 
\begin{proof}
The proof stems from the following observations. By Lemma \ref{lem:optimal_survives}, arm $a^{\star}$ is with high probability never eliminated. After $N(a)$ rounds, with high probability, the lower confidence interval of arm $a^{\star}$ is above the upper confidence interval of arm $a$. This comes from the fact that, after $N(a)$ plays of arm $a$ (and also of arm $a^{\star}$ since it is not eliminated), the empirical stochastic mean of $a^{\star}$ is, with high probability, at most $\nicefrac{\Delta(a)}{6}$ below its actual mean and similarly the empirical stochastic mean of arm $a$ is at most $\nicefrac{\Delta(a)}{6}$ above its actual mean. Since the corruptions are upper bounded by $C$, they can only contribute to a decrease in the average empirical (corrupted) means by at most $\nicefrac{\Delta(a)}{6}$ which is not enough to circumvent the gap $\Delta(a)$.

More formally, let $\widetilde{\mu}_{\s}(a)$ and $\widetilde{\mu}(a)$ denote the empirical means of the stochastic part of the rewards and the corrupted rewards respectively after $N(a)$ plays of arm $a$. By the same Hoeffding inequality as in the proof of the previous lemma, with probability at least $1-\delta$, it holds that $\abs{\widetilde{\mu}_{\s}(a)-\mu(a)}\leq \sqrt{\frac{\log\prn*{\nicefrac{2KT}{\delta}}}{N(a)}}$. Therefore, with the same probability, after $\frac{36\cdot\log(\nicefrac{2KT}{\delta})}{\Delta(a)^2}$ plays for both arm $a$ and $a^{\star}$: $\widetilde{\mu}_{\s}(a^{\star})-\mu(a^{\star})\leq \frac{\Delta(a)}{6}$ and $\mu(a)-\widetilde{\mu}_{\s}(a)\leq \frac{\Delta(a)}{6}$. 

The absolute corruption is at most $C$ therefore
$\widetilde{\mu}(a^{\star})\geq\widetilde{\mu}_{\s}(a^{\star})-\frac{C}{N(a)}$ and
$\widetilde{\mu}(a)\leq\widetilde{\mu}_{\s}(a)+\frac{C}{N(a)}$. By the choice of $N(a)$, we have $\frac{C}{N(a)}\leq \frac{\Delta(a)}{6}$. Combining with the above argument, this also implies that the widths are upper bounded by 
$\wid(a)\leq \frac{\Delta(a)}{3}$ and $\wid(a^{\star})\leq \frac{\Delta(a)}{3}$.

Combining the above with the fact that the actual mean of $a^{\star}$ is $\Delta(a)$ higher than the one of $a$, i.e. $\mu(a^{\star})-\mu(a)=\Delta(a)$, we establish \begin{align*}\widetilde{\mu}(a^{\star})-\widetilde{\mu}(a)-\wid(a)-\wid(a^{\star})&\geq 
\widetilde{\mu}_{\s}(a^{\star})-\widetilde{\mu}_{\s}(a)-2\cdot\frac{C}{N(a)}-\wid(a)-\wid(a^{\star})\\
&> \mu(a^{\star})-\mu(a)-2\cdot \frac{\Delta(a)}{6}-\frac{\Delta(a)}{3}-\frac{\Delta(a)}{3}>0
\end{align*}
As a result arm $a$ becomes eliminated after $N(a)$ plays if it is not already eliminated before.
\end{proof}

\textbf{Theorem \ref{thm:known_corruption} (restated) } 
If $C$ is a valid upper bound for the total corruption then arm elimination with $\wid(a) = \sqrt{\frac{\log\prn*{\nicefrac{2KT}{\delta}}}{n(a)}}+\frac{C}{n(a)}$ has regret $\bigO\prn*{\sum_{a \neq a^\star}\prn*{\frac{\log\prn*{\nicefrac{KT}{\delta}}+C}{\Delta(a)}}}$ with probability $1-\delta$.
\begin{proof}
The proof follows the classical stochastic bandit argument of measuring the regret caused by each arm $a\neq a^{\star}$ as a function of its gap $\Delta(a)$ and the number of times $N(a)$ it is played as established by Lemma \ref{lem:bound_suboptimal_plays}.

For simplicity of presentation, we first provide the pseudo-regret guarantee. Pseudo-regret compares the expected performance of the algorithm to the expected performance one would have had, had they selected $a^{\star}$ throughout the whole time horizon. The expected performance when one uses $a^{\star}$ is $\mu(a^{\star})$. The loss compared to that every time $a\neq a^{\star}$ is used instead is equal to its gap $\Delta(a)$. As a result, the expected contribution to pseudo-regret from suboptimal arm $a\neq a^{\star}$ is equal to $N(a)\cdot \Delta(a)$. Lemma \ref{lem:bound_suboptimal_plays} establishes that with probability $1-\delta$ any suboptimal arm $a$ is played at most $N(a)=\frac{36\log(\nicefrac{2KT}{\delta})+6C}{\Delta(a)^2}$ times. Each play of the suboptimal arm causes pseudo-regret of $\Delta(a)$. Multiplying the times by the expected regret per time the guarantee (which equals to the gap) and setting the failure probability $\delta$ to be some inverse polynomial of the time horizon $T$ to ensure that the expected regret due to the bad event is at most a constant leads to the pseudo-regret guarantee.

To turn the above into a high-probability guarantee, we need to show that the regret incurred during the steps that we pull arm $a$ is not significantly higher than the expectation (therefore bounding the resulting variance). By the Hoeffding inequality of Lemma \ref{lem:optimal_survives}, the empirical cumulative reward of arm $a$ is, with high probability, at most $\sqrt{N(a)\log(\nicefrac{2KT}{\delta})}$ less than its expectation. The same holds for arm $a^{\star}$ for these steps (its realized performance is at most this much more than its expectation). The probability that these statements do not hold for some arm or some time is at most $\delta$.

Regarding arms $a\neq a^{\star}$, the $\sqrt{N(a)\log(\nicefrac{2KT}{\delta})}$ term can be upper bounded by $\bigO\prn*{N(a)\Delta(a)}$ by the definition of $N(a)$:
\begin{align*}
\sqrt{N(a)\log\prn*{\nicefrac{2KT}{\delta}}}\leq N(a)\cdot \sqrt{\frac{\log(\nicefrac{2KT}{\delta})}{N(a)}}\leq N(a)\cdot  \Delta(a)\sqrt{\frac{\log\prn*{\nicefrac{2KT}{\delta}}}{36\log\prn*{\nicefrac{2KT}{\delta}}+6C}}\leq N(a)\Delta(a)
\end{align*}
Regarding arm $a^{\star}$, let $a'$ be the arm with the smallest gap. By Lemma \ref{lem:optimal_survives}, $a^{\star}$ never gets eliminated but it is not necessarily the ex post optimal arm. In fact some other arm with $\Delta(a)\leq \sqrt{\nicefrac{1}{T}}$ may be the ex post optimal arm (arms with higher gap are with high probability not the ex post optimal arm by an analogous argument as in Lemma \ref{lem:bound_suboptimal_plays}. However, by the same argument as above arm $a^{\star}$ is with high probability at most $N(a')\cdot\Delta(a')$ below its expectation and the ex post optimal arm is at most this much above its expectation. This gives a bound of $N(a')\Delta(a')$ that is caused by the case where $a^{\star}$ is not the ex post optimal arm.

Therefore the actual regret from times that arm $a$ is played is at most $2N(a)\Delta(a)$ where the one term comes from the expectation and the other from the aforementioned bounds on the variance. The corruption can increase any cumulative reward by at most $C$ which is already existing in the regret bound. Replacing $N(a)$ by Lemma \ref{lem:bound_suboptimal_plays}, we obtain the high-probability guarantee. Note that the failure probabilities of the two lemmas are coupled as they correspond to the same bad events.
\end{proof}

\section{Supplementary material on Section \ref{ssec:robust_known}}
\label{app:sec_robust_known}

In this section, we provide the proof of Theorem \ref{thm:stoch_known corruption}. To handle the corruption, we bound with high probability the total corruption experienced by the slow active arm elimination instance $\S$ (Lemma~\ref{lem:actual_corruption_high_prob}). To deal with an adaptive adversary, we need a martingale concentration inequality; specifically we apply a Bernstein-style inequality introduced in \cite{BeygelzimerLLRS11} (Lemma \ref{lem:improved_martingale_inequality}).

\begin{lemma}[Lemma 1 in \cite{BeygelzimerLLRS11}]\label{lem:improved_martingale_inequality}
Let $X_1,\dots,X_T$ be a sequence of real-valued random numbers. Assume, for all $t$, that $X_t\leq R$ and that $\mathbb{E}[X_t| X_1,\dots,X_{t-1}]=0$. Also let 
$$
V=\sum_{t=1}^T \mathbb{E}[X_{t}^2|X_1,\dots,X_{t-1}].
$$
Then, for any $\delta>0$:
$$
\mathbb{P}\brk*{\sum_{t=1}^T X_t > R\ln(1/\delta)+\frac{e-2}{R}\cdot V}\leq \delta
$$
\end{lemma}

\textbf{Lemma \ref{lem:actual_corruption_high_prob} (restated)}
In Algorithm \ref{alg:known}, the slow active arm elimination algorithm $\S$ observes, with probability at least $1-\delta$, corruption of at most $\ln(\nicefrac{1}{\delta})+3$ during its exploration phase (when picked with probability $\nicefrac{1}{C}$).
\begin{proof}
The first observation is that the expected corruption encountered by algorithm $\S$ is at most a constant (total corruption of $C$ encountered with probability $\nicefrac{1}{C}$). The rest of the proof focuses on bounding the variance of this random variable (actual corruption encountered by the layer). Crucially, since we want to allow the adversary to be adaptive, we should not assume independence across rounds but only conditional independence (conditioned on the history) and this is why some more involved concentration inequality is necessary. Therefore we create a martingale sequence (actual corruption minus expected corruption) and apply a Bernstein-style concentration inequality.

Let $Z_a^t$ be the corruption that is observed by the exploration phase of the algorithm if arm $a$ is selected. For every round $t$, if adversary selects corruption $C_a^t$ then $Z_a^t$ is therefore a random variable equal to $C_a^t$ with probability $\nicefrac{1}{C}$ and $0$ otherwise. Given that the adversary is adaptive and may select the corruptions based on the realizations of the previous rounds, we need to use an appropriate concentration inequality. We use a Bernstein-style inequality, introduced in \cite{BeygelzimerLLRS11}  (Lemma \ref{lem:improved_martingale_inequality}). Initially we resolve the randomness conditioning on $\ell=\S$ (the slow algorithm is selected). Since active arm elimination is deterministic, conditioned on selecting algorithm $\S$, the selected arm is deterministic. Let $a({\S},t)$ be the arm that would be selected if $\ell=\S$ (which happens with probability $\nicefrac{1}{C}$). The martingale sequence is now $$X_t=Z_{a(\S,t)}^t-\mathbb{E}\brk*{Z_{a(\S,t)}^t\mid \mathcal{H}(1:{t-1})}$$ where $\mathcal{H}(1:t)$ corresponds to the history up to round $t$. Note that
\begin{align*}
\mathbb{E}\brk*{X_t^2|X_1,\dots,X_{t-1}} &= \frac{1}{C}\prn*{C_{a(\S,t)}-\frac{C_{a(\S,t)}}{C}}^2+\frac{C-1}{C}\prn*{\frac{C_{a(\S,t)}}{C}}^2 \\
&= \frac{\prn*{C_{a(\S,t)}}^2}{C}\prn*{\frac{C-1}{C}}^2+\frac{C-1}{C}\prn*{\frac{C_{a(\S,t)}}{C}}^2\leq 2\cdot\frac{C_{a(\S,t)}}{C}.
\end{align*}
The last inequality holds as $C_{a(\S,t)}^t\in[0,1]$ and $C_{a(\S,t)}\leq C$ by the definition of $C$.

Therefore, summing over all the rounds, $$V=\sum_t \mathbb{E}\brk*{X_t^2|X_1,\dots,X_{t-1}}\leq \sum_t 2\frac{C_{a(\S,t)}}{C}\leq \frac{2}{C}\cdot \prn*{\sum_t \max_a C_a^t}\leq 2.$$ 

A trivial upper bound of $\abs{X_t}$ is $R=1$, since the rewards are in $[0,1]$. Applying Lemma \ref{lem:improved_martingale_inequality}, we show that, w.p. $1-\delta$:
\begin{align*}
\sum_t X_t \leq \ln(\nicefrac{1}{\delta})+2(e-2)\leq \ln(\nicefrac{1}{\delta})+2
\end{align*}
The lemma then follows by adding the expected corruption of $\mathbb{E}\brk*{\sum_t Z_{a(\S,t)}\mid \mathcal{H}(1:t-1)}\leq 1$ and therefore obtaining the bound of the statement on the corruption experienced:
\begin{align*}
\sum_t Z_{a(\S,t)}^t = \sum_t X_t +\mathbb{E}\brk*{\sum_t Z_{a(\S,t)}\mid \mathcal{H}(1:t-1)}\leq \ln(\nicefrac{1}{\delta})+3.
\end{align*}
\end{proof}

\textbf{Theorem \ref{thm:stoch_known corruption} (restated)}
Algorithm \ref{alg:known} run with widths
$\wid^{\S}(a)=\sqrt{\frac{\log(\nicefrac{8KT}{\delta})}{n^{\S}(a)}}+\frac{2\log(\nicefrac{8KT}{\delta})}{n^{\S}(a)}$ and  $\wid^{\F}(a)=\sqrt{\frac{\log(\nicefrac{8KT}{\delta})}{n^{\F}(a)}}$ has $\bigO\prn*{\sum_{a\neq a^{\star}}\frac{\log\prn*{\nicefrac{KT}{\delta}}}{\Delta(a)}}$  for the stochastic case and $\bigO\prn*{K\cdot C\cdot\sum_{a\neq a^{\star}}\frac{\prn*{\log\prn*{\nicefrac{KT}{\delta}}}^2}{\Delta(a)}}$ for the $C$-corrupted case with probability at least $1-\delta$. 

\begin{proof}
For the stochastic case, the bound follows via standard stochastic bandit arguments (similarly to the proof of Theorem \ref{thm:known_corruption} with $C=0$) as for each of the two active arm elimination algorithms we incur, with probability $1-\delta_{\ell,\s}$ regret $\bigO\prn*{\sum_{a\neq a^{\star}}\frac{\prn*{\log\prn*{\nicefrac{2KT}{\delta_{\ell,\s}}}}}{\Delta(a)}}$ where $\delta_{\ell,\s}=\nicefrac{\delta}{4}$ is the failure probability of inequality \eqref{eq:main_hoeffding_known}, which governs the results in Lemmas \ref{lem:optimal_survives} and \ref{lem:bound_suboptimal_plays}, for each of $\ell\in\crl*{\F,\S}$.

The most interesting case is the $C$-corrupted setting. Let $\delta_{\S,C}=\nicefrac{\delta}{4}$ be the failure probability in Lemma \ref{lem:actual_corruption_high_prob}. By Lemma \ref{lem:actual_corruption_high_prob}, with probability at least $1-\delta_{\S,C}$, the actual corruption experienced by the slow active arm elimination algorithm is at most $\ln(1/\delta)+3$ which is less than $2\log(\nicefrac{2KT}{\delta})$ for non-trivial values of $K$ and $T$. Therefore we can apply the analysis of Theorem \ref{thm:known_corruption} with corruption level at least $2\log(\nicefrac{2KT}{\delta_{\S,C}})$ and get a handle on the actual regret coming from the slow active arm elimination algorithm. 

What is left is to bound the regret coming from the fast active arm elimination algorithm. Towards this goal, we bound the number of times that a suboptimal arm is played in the fast active arm elimination by the expected time that it remains active at the slow active arm elimination. By Lemma \ref{lem:bound_suboptimal_plays}, arm $a$ is played in the slow active arm elimination, with probability at least $1-\nicefrac{3\delta}{4}$, at most
$$N_{\S}(a)=\frac{16\log(\nicefrac{2KT}{\delta_{\S,\s}})+2\log(\nicefrac{2KT}{\delta_{\S,C}})}{\Delta(a)^2}\leq \frac{18\log(\nicefrac{8KT}{\delta})}{\Delta(a)^2}.$$

Having a bound on the number of plays of the arm in the slow active arm elimination instance, we use this to bound the number of plays in the fast active arm elimination instance. In expectation, this is at most $K\cdot C\cdot N_{\S}(a)$ times as every move in the slow active arm elimination occurs with probability $\nicefrac{1}{C}$ and, at least $\nicefrac{1}{K}$ of these moves are plays of $a$ while it is still active. Since every time arm $a$ is played it incurs pseudo-regret $\Delta(a)$, this provides the pseudo-regret guarantee. 

To obtain a high probability guarantee, let $\delta_m=\nicefrac{\delta}{4KT}$ and observe that with probability at least $1-\delta_m$, we make one move at the slow arm elimination algorithm every $\bigO\prn*{C\log\prn*{\nicefrac{1}{\delta_m}}}$ moves at the fast arm elimination algorithm. This can be seen by thinking the following process: One tosses coins with bias $p=\nicefrac{1}{C}$ until she observes heads for the first time (heads is the $p$-biased event). After $M$ tosses of the coins the probability that no heads have arrived is at most $\prn*{1-p}^M$. To ensure that this is less than $\delta_m$, we need to wait $M\geq\frac{\log\prn*{\nicefrac{1}{\delta_m}}}{\log(\frac{1}{1-p})}$, which is achieved by $M=\frac{\log(\nicefrac{1}{\delta_m})}{\nicefrac{p}{(1-p)}}$.

By union bound on the failure probabilities for each of those draws, we get that with failure probability $\delta_e=K\cdot N_{\S}(a)\cdot \delta_m\leq \nicefrac{\delta}{4}$ (since $N_{\S}(a)\leq T$ as it is at most the time horizon), arm $a$ gets inactivated in $\F$ after
$$N_{\F}(a)=K\cdot N_{\S}(a)\cdot C\cdot\log(\nicefrac{1}{\delta_e})=\frac{18\cdot C\cdot K\cdot  \prn*{\log(\nicefrac{8KT}{\delta})}^2}{\Delta(a)^2}.$$

The last part is to prove that the regret experienced throughout those rounds is not too large. This follows by the two applications of Hoeffding inequality as before for arms $a$ and $a^{\star}$, analogously to Theorem \ref{thm:known_corruption}. Combining the above arguments the theorem follows. The total failure probability of the guarantee is $\delta_{\S,\s}+\delta_{\S,C}+\delta_{\F,\s}+\delta_e\leq \delta$.
\end{proof}

\section{Supplementary material on Section \ref{ssec:robust_agnostic}}
\label{app:sec_robust_agnostic}

\textbf{Theorem \ref{thm:agnostic} (restated)}
Algorithm \ref{alg:agnostic} which is agnostic to the coruption level $C$, when run with widths
$\wid^{\ell}(a)=\sqrt{\frac{\log(\nicefrac{4KT\cdot \log T}{\delta})}{n^{\ell}(a)}}+\frac{\log(\nicefrac{4KT\cdot\log T}{\delta})}{n^{\ell}(a)}$ has regret: $$\bigO\prn*{\sum_{a\neq a^*}\frac{K\cdot C\log\prn*{\nicefrac{KT}{\delta}}+\log\prn*{T}}{\Delta(a)}\cdot\log\prn*{\nicefrac{KT}{\delta}}}.$$
\begin{proof}
The proof follows similar arguments to the proof of Theorem \ref{thm:stoch_known corruption}. Specifically, for the layers that are above the corruption level $C$, by using the standard arguments described in Theorem \ref{thm:known_corruption}, we establish a $\frac{\log(\nicefrac{2KT}{\delta_{\ell,\s}})}{\Delta(a)}$ bound on the regret caused by any suboptimal arm $a$,
with failure probability $\delta_{\ell,\s}=\nicefrac{\delta}{2\log T}$. Since there are $\log(T)$ such levels, the regret coming from these layers is upper bounded by the second term of the theorem with failure probability$\nicefrac{\delta}{2}$. 

For the layers $\ell$ that are not tolerant to the corruption, i.e. $2^{\ell}>C$, we apply the same argument as in the proof of Theorem \ref{thm:stoch_known corruption} and bound their regret via the number of plays they are played by the minimum layer that is robust to corruption $\ell^{\star}=\argmin_{\ell}\brk*{2^{\ell}>C}$. Similarly as in the proof of the theorem we upper bound the number of plays $N_{\ell^{\star}}(a)$ of each suboptimal arm $a$ at this layer (by exactly the same arguments), then bound the number of plays in the suboptimal layer via the same coin toss process as in that proof and, last bound the regret they incur during this part. Since we do not know the amount of corruption in advance (and this amount is adaptively selected), we also need to take a union bound on the number of layers so that the guarantee on $N_{\ell^{\star}}(a)$ holds for all layers simulataneously if they end up being correct; we therefore repeat the arguments in Theorem \ref{thm:stoch_known corruption} with $\delta_{\ell,C}=\nicefrac{\delta}{2\log T}$ and $\delta_m \leq \nicefrac{\delta}{2KT \log T}$.

Last, we note that, since we used powers of $2$ to increase the corruption among layers, the fact that we did not apply the arguments of Theorem \ref{thm:stoch_known corruption} with the exact $C$ but instead used a $C'$ such that $C<C'<2C$ causes just an extra constant factor on the regret.
\end{proof}

\section{Supplementary material on Section \ref{sec:lower_bound}}
\label{app:sec_lower_bound}
\textbf{Theorem \ref{thm:lower_bound_1} (restated)} Consider a multi-armed bandits algorithm that has the property
that for any stochastic input in the two arm setting, it has pseudo-regret
bounded by $c \log(T) / \Delta$,
where $\Delta = \abs{\mu_1 - \mu_2}$. For any $\epsilon, \epsilon' \in (0,1)$,
there is a corruption level $C$ with $T^\epsilon < C < T^{\epsilon'}$ and a
$C$-corrupted instance such that with constant probability the regret is
$\Omega(C)$. 
\begin{proof} The proof follows a sequence of steps.

{\bf Step 1: Analyze behavior in the stochastic case.}
Fix a constant $\Delta \leq 1/6$ and observe how the algorithm behaves
for the stochastic input that has Bernoulli arms of means
$(\mu_1, \mu_2) = (\half - \Delta, \half)$.
Since in that setting the expected regret is the same as $\Delta \cdot \E [T_1]$
where $T_1$ is the number of pulls of arm $1$, it follows that  
$\E [T_1] \leq c \log(T) /
\Delta^2$.

{\bf Step 2: find a large interval that is hit with at most constant
  probability.} We divide the space between $T^\epsilon$ and $T^{\epsilon'}$ into
  $\bigO( \log(T) / (\epsilon' - \epsilon) )$ intervals  $I_i = [3^{i-1}
  T^\epsilon,
  3^{i} T^\epsilon)$ such that  size of each interval is
twice the size of all the previous intervals combined.
For each interval $i$, let $T_{1,i}$ be the number of times
that arm $1$ is pulled in the $i$-th interval. Then, there exists an interval $I_i = [C, 3C)$   such that $\E[T_{1,i}] \leq \tilde{c} := \bigO(1 /
  [ (\epsilon' - \epsilon) \Delta^2 ])$.

{\bf Step 3: create an adversary that forces a lot of regret in interval $i$.}
  The adversary is quite simple: for the first $C$ steps, the arms are Bernoulli
  with means $(\half - \Delta, \half)$ and for the remaining timesteps, the arms are
  Bernoulli with means $(\half + \Delta, \half)$.

We use $\E$ and $\P$ to refer to the probability law whe inputs are drawn
with respect to $(\half - \Delta, \half)$ in all timesteps and $\E'$ and $\P'$ to
refer to the probability law when the input is according to $(\half - \Delta,
\half)$ in the first $K$ steps and according to $(\half + \Delta, \half)$ onwards.

{\bf Step 4: With constant probability arm $1$ is pulled a constant number of
times in $I_i$ under both $\P$ and $\P'$.} Under the probability law $\P$,
this follows directly from Markov's inequality: $\tilde{c} = \E[T_{1,i}] \geq
2 \tilde{c} \P[ T_{1,i} \geq 2 \tilde{c}]$, so:  $\P[ T_{1,i} \leq 2 \tilde{c}]
\geq \half$.

Denote by $A$ the event that $T_{1,i} \leq 2 \tilde{c}$. We want to argue that
$\P'[A]$ is also constant. In order to do that, let $Z = (Z_1, Z_2, \hdots,
Z_{2\tilde{c}})$ be a vector storing in $Z_s$ the reward of arm $1$ in the
$s$-th time it is pulled in interval $I_i$. Notice that in both the stochastic
and corrupted scenarios if the learner observes the same values of $Z$ she acts the exact same way. Therefore, if we condition on $Z$, the probability that
she ends up pulling arm $1$ for more than $2\tilde{c}$ times is exactly the
same. In other words:
$$\P[A \vert Z] = \P'[A \vert Z]$$
Therefore:
  $$\P'[A] = \sum_z \P'[Z=z] \cdot \P[A \vert Z=z] \geq
  \left(\frac{\half-\Delta}{\half+\Delta}\right)^{2 \tilde{c}} \P'[Z=z] \P[A
  \vert Z=z] \geq \half \cdot   \left(\frac{\half-\Delta}{\half+\Delta}\right)^{2 \tilde{c}}$$
which is a constant.

{\bf Step 5: concentration bounds for the regret incurred in each interval.} We now define an event $B$ that occurs with probability $1-o(1)$ that captures all the concentration bounds we need for the proof. First, we require arm $1$ to be the optimal arm. Let $r^t(i)$ be the reward of arm $i$ in time step $t$. We know that $\E'[\sum_t r^t(1)] = \half T + \Delta(T-2C) $ and $\E'[\sum_t r^t(2)] = \half T$. Since all the rewards are independent we can use
  the Hoeffding bound to bound the probability $\P'(\sum_t \Delta_t < 0)$ where
  $\Delta_t = r^t(2) - r^t(1)$:
  $$\P'\left( \sum_t r^t(1) < \sum_t r^t(2) \right) \leq \P'\left( \left\vert
  \sum_t \Delta_t - \E'\Delta_t \right\vert > \Delta(T-2C) \right) \leq 2
  \exp\left( -\frac{T \Delta^2}{2} \left( 1-2\frac{C}{T}\right) \right) = o(1)$$

  \newcommand{\ho}[1]{[#1)}

  Now we establish some concentration on the regret that the learner
  achieves with respect to arm $1$ in the intervals $\ho{1,C}$, $\ho{C,3C}$ and
  $\ho{3C,T}$. We note that if the learner pulls arm $1$, she does not incur  any
  regret. If she pulls arm $2$, she incurs regret $\Delta_t = r^t(2) - r^t(1)$
  which can be positive or negative. To compute regret with respect to arm $1$
  in each of those intervals, we sample $\Delta_t$ every time that the arm $2$ is pulled.

  {\bf Step 5a: interval $\ho{1,C}$}. In this interval, $\E'\Delta_t = -\Delta$, so
  $\sum_{t \in [1,C)} \Delta_t =
  -C\Delta$. If $Y$ is the number of times arm $1$ is pulled, then the regret is
  given by $\sum_{s=1}^Y \Delta_s$ where in the previous expression we abuse
  notation and mean by $\Delta_s$ the regret in the $s$-th time the arm $2$ is
  pulled instead of the regret in the $t$-th period. Therefore:
  $$\P'\left( \sum_{s=1}^Y \Delta_s < -1.1 \Delta C \right) \leq
  \P'\left( \min_{t\leq C}\sum_{s=1}^t \Delta_s < -1.1 \Delta C \right) \leq
  \sum_{t=1}^C \P'\left( \sum_{s=1}^t \Delta_s < -1.1 \Delta C
  \right) $$
  We then use the Hoeffding bound in the last expression and get:
  $$\sum_{t=1}^C 2 \exp\left( -\half t \left( \frac{1.1 \Delta C - \Delta t}{t}
  \right)^2 \right) \leq C \cdot \exp\left( - \half (0.1 \Delta C)^2 \right) =
  o(1)$$

  {\bf Step 5b: interval $\ho{C,3C}$}. In this interval, $\E' \Delta_t = \Delta$,
  so using the same bound as before, we get
  $$\P' \left( \sum_{t\in\ho{C,3C}} \Delta_t < 1.9\Delta C \right) \leq
  2\exp\left(
  - C (0.1 \Delta)^2 \right) = o(1)$$

  {\bf Step 5c: interval $[3C,T]$} In this interval, pulling arm $2$ has again
  positive expected regret. We use the same technique used in 5a to argue that
  she cannot obtain large negative regret with high probability:
  Let $Y$ be the number of times arm $2$ is pulled in that interval and again we
  abuse notation and let $\Delta_s$ be the difference in rewards in the $s$-th
  time the arm is pulled. Then:
  $$\P'\left( \sum_{s=1}^Y \Delta_t < - \log T\right) \leq \P'\left(\min_{t \leq
  T} \sum_{s=1}^t \Delta_t < -\log T\right) \leq \sum_{t=1}^T
  \P' \left(\sum_{s=1}^t \Delta_t < -\log T\right)  $$
  For $t=1..\log T$ this probability is zero, since $\Delta_t \geq -1$. Now, for
  larger $T$, we can use the standard Chernoff bound:
  $$ \sum_{t=\log T}^T    \P' \left(\sum_{s=1}^t \Delta_t < -\log T\right)  \leq
    2 T \exp \left( -\half \log^2(T) \Delta \right) = o(1)
  $$

  Now all the concentration bounds have been established we define the event $B$
  to be the event where all those concentration bounds hold. More precisely, $B$
  is the event where the following four things happen: (a) empirically arm $1$
  is better than arm $2$; (b) in interval $\ho{1,C}$, the regret of the learner
  is at least $-1.1\Delta C$; (c) in interval $\ho{C,3C}$, the difference
  between the total rewards of both arms is at least $1.9\Delta C$; and (d) the
  regret of the learer in interval $[3C, T]$ is at least $-\log T$. By the
  discussion in step 5, we know that $\P'(B) = 1-o(1)$.

{\bf Step 6: putting it all together.} Since $\P'(A) = \Omega(1)$ and
$\P'(B) = 1-o(1)$, then by the union bound, $\P'(A \text{ and } B) \geq \P'(A) -
o(1) = \Omega(1)$. Now, we need to argue that in the constant probability event $(A \text{
  and } B)$, the regret of the learner is at least $\Omega(C)$.

We simply sum the regret of the learner in each of the intervals. For intervals
$\ho{1,C}$ and $[3C,T]$ we can use the bounds computed in steps 5a andn 5c
directly. For interval $\ho{C,3C}$, we note that conditioned on $A$, the learner
probes arm $1$ a constant number of times, so his total regret differs from the
regret by pulling arm $2$ in all iterations by at most a constant, therefore the
total regret can be bounded by:
$$-1.1\Delta C + (1.9 \Delta C - 4 \tilde{c}) - \log(T) = \Omega(\Delta C)$$
\end{proof}

We can adapt the argument to provide a bound on the expected positive regret
$\E[\regret^+]$ where $x^+ = \max\{x,0\}$. Note that the high probability bounds provided also
imply a bound on the expected positive regret.\\

\textbf{Theorem \ref{thm:lower_bound_2}.} If a multi-armed bandits algorithm that has the property
that for any stochastic input in the two arm setting, it has pseudo-regret
bounded by $c \log^{1+\alpha}(T) / \Delta$ for $\alpha < 1$.
For any $\epsilon, \epsilon' \in (0,1)$,
there is a corruption level $C$ with $T^\epsilon < C < T^{\epsilon'}$ and a
$C$-corrupted instance such that
$\E[\regret^+] = \Omega(T^{\epsilon - \delta})$ for all $\delta > 0$.
\begin{proof}
Modify the proof of Theorem \ref{thm:lower_bound_1} as follows. Define
$\tilde{c} = \bigO(\log^\alpha(T) / [(\epsilon' - \epsilon) \Delta^2])$ and again select
an interval such that $\E[T_{i,1}] \leq \tilde{c}$. Event $A$ is defined in
the same way. By Markov's inequality:
$\P[A] \geq \nhalf$ and
$$\P'[A] \geq \frac{1}{2} \cdot  
\left(\frac{\half-\Delta}{\half+\Delta}\right)^{2\tilde{c}}
= \exp\prn*{-\bigO(\log^{\alpha}(T))} $$
Step 5 remains unchanged and in step 6 note that $\P'(B) \ll \P'(A)$ since $\alpha < 1$,
so $\P'(A \text{ and } B) = \exp\prn*{-\bigO(\log^{\alpha}(T))}$. Therefore, with probability
at least $\exp\prn*{-\bigO(\log^{\alpha}(T))}$ the regret is at least $\Omega(C) = \Omega(T^{\epsilon})$
and therefore, $\E[\regret^+] = \Omega(T^{\epsilon} \cdot \exp\prn*{-\bigO(\log^{\alpha}(T))}) =\Omega(T^{\epsilon-\delta}) $
for all $\delta > 0$.
\end{proof}

\end{document}